\documentclass{article}

\usepackage[nonatbib, preprint]{neurips_2020}

\usepackage[utf8]{inputenc} 
\usepackage[T1]{fontenc}    
\usepackage[dvipsnames]{xcolor}
\definecolor{mydarkblue}{rgb}{0,0.08,0.65}
\usepackage[colorlinks=true,linkcolor=mydarkblue,citecolor=mydarkblue]{hyperref}%
\usepackage{url}            
\usepackage{booktabs}       
\usepackage{amsfonts}       
\usepackage{nicefrac}       
\usepackage{microtype}      

\title{Neural Manifold Ordinary Differential Equations}

\author{
    Aaron Lou*, Derek Lim*, Isay Katsman*, Leo Huang*, Qingxuan Jiang\\
    Cornell University\\
    \texttt{\{al968, dl772, isk22, ah839, qj46\}@cornell.edu}\\
    \And
    Ser-Nam Lim\\
    Facebook AI\\
    \texttt{sernam@gmail.com}\\
    \And
    Christopher De Sa\\
    Cornell University\\
    \texttt{cdesa@cs.cornell.edu}
}

\usepackage{amsmath, amssymb, amsthm, bbm, mathtools, commath, amsfonts, enumitem}
\usepackage{xcolor}
\usepackage{subfig}
\usepackage{graphicx}
\usepackage{wrapfig} 
\usepackage{colortbl}
\usepackage{multirow}

\graphicspath{ {./imgs/} }
\usepackage[ruled,vlined]{algorithm2e}

\newtheorem{thm}{Theorem}[section]
\newtheorem{prop}[thm]{Prop}

\newtheorem*{remark}{Remark}

\newtheorem*{lemma}{Lemma}

\newcommand{\mc}{\mathcal}
\newcommand{\mrm}{\mathrm}

\DeclareMathOperator{\arccosh}{arccosh}

\newcommand{\inn}[1]{\left\langle#1\right\rangle}
\newcommand{\R}{\mathbb{R}}

\newcommand{\N}{\mathcal{N}}

\newcommand{\paren}[1]{\left(#1\right)}

\newcommand{\sqbrac}[1]{\left[#1\right]}

\renewcommand{\abs}[1]{\left|#1\right|}
\renewcommand{\norm}[1]{\left\|#1\right\|}

\newcommand{\parderiv}[2]{\frac{\partial #1}{\partial #2}}
\newcommand{\M}{\mathcal{M}}
\newcommand{\mbf}{\mathbf}

\renewcommand{\mrm}{\mathrm}
\newcommand{\RR}{\mathbb{R}}
\newcommand{\tabcent}{\multicolumn{1}{c}}

\begin{document}

\maketitle

\begin{abstract}
    To better conform to data geometry, recent deep generative modelling techniques adapt Euclidean constructions to non-Euclidean spaces. In this paper, we study normalizing flows on manifolds. Previous work has developed flow models for specific cases; however, these advancements hand craft layers on a manifold-by-manifold basis, restricting generality and inducing cumbersome design constraints. We overcome these issues by introducing Neural Manifold Ordinary Differential Equations, a manifold generalization of Neural ODEs, which enables the construction of Manifold Continuous Normalizing Flows (MCNFs). MCNFs require only local geometry (therefore generalizing to arbitrary manifolds) and compute probabilities with continuous change of variables (allowing for a simple and expressive flow construction). We find that leveraging continuous manifold dynamics produces a marked improvement for both density estimation and downstream tasks.
\end{abstract}

\section{Introduction}\label{sec:intro}

\begin{wrapfigure}{R}{.35\textwidth}
  \vspace{-5pt}
  \includegraphics[width=.35\textwidth]{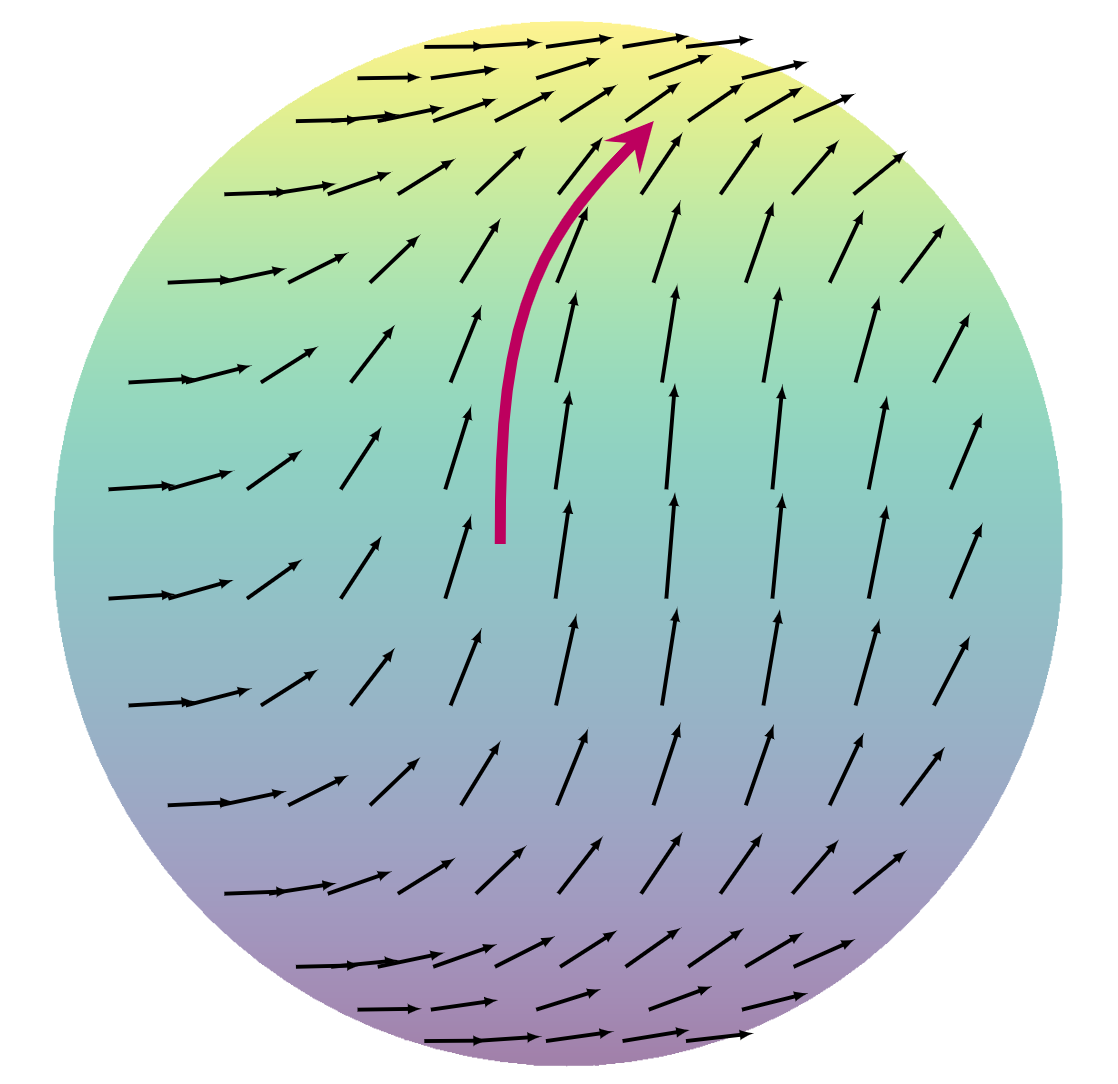}
  \caption{A manifold ODE solution for a given vector field on the sphere.}
  \label{fig:mode}
  \vspace{-5pt}
\end{wrapfigure}
\let\thefootnote\relax\footnotetext{* indicates equal contribution}
Deep generative models are a powerful class of neural networks which fit a probability distribution to produce new, unique samples. While latent variable models such as Generative Adversarial Networks (GANs) \cite{Goodfellow2014GenerativeAN} and Variational Autoencoders (VAEs) \cite{Kingma2014AutoEncodingVB} are capable of producing reasonable samples, computing the exact modeled data posterior is fundamentally intractable. By comparison, normalizing flows \cite{Rezende2015VariationalIW} are capable of learning rich and tractable posteriors by transforming a simple probability distribution through a sequence of invertible mappings. Formally, in a normalizing flow, a complex distribution $p(x)$ is transformed to a simple distribution $\pi(z)$ via a diffeomorphism $f$ (i.e. a differentiable bijective map with a differentiable inverse) with probability values given by the change of variables:
\begin{equation*}\label{eqn:logProbUpdate}
    \log p(x) = \log \pi(z) - \log \det \abs{\parderiv{f^{-1}}{z}}, \qquad z = f(x).
\end{equation*}
To compute this update efficiently, $f$ must be constrained to allow for fast evaluation of the determinant, which in the absence of additional constraints takes $\mathcal{O}(D^3)$ time (where $D$ is the dimension of $z$). Furthermore, to efficiently generate samples, $f$ must have a computationally cheap inverse. Existing literature increases the expressiveness of such models under these computational constraints and oftentimes parameterizes $f$ with deep neural networks \cite{Germain2015MADEMA, Dinh2017DensityEU, Kingma2018GlowGF, Durkan2019NeuralSF}. An important recent advancement, dubbed the Continuous Normalizing Flow (CNF), constructs $f$ using a Neural Ordinary Differential Equation (ODE) with dynamics $g$ and invokes a continuous change of variables which requires only the trace of the Jacobian of $g$ \cite{Chen2018NeuralOD, Grathwohl2019FFJORDFC}.

Since $f$ is a diffeomorphism, the topologies of the distributions $p$ and $\pi$ must be equivalent. Furthermore, this topology must conform with the underlying latent space, which previous work mostly assumes to be Euclidean. However, topologically nontrivial data often arise in real world examples such as in quantum field theory \cite{Wirnsberger2020TargetedFE}, motion estimation \cite{Feiten2013RigidME}, and protein structure prediction~\cite{Hamelryck2006SamplingRP}.

In order to go beyond topologically trivial Euclidean space, one can model the latent space with a \textit{smooth manifold}. An $n$-dimensional manifold\footnote{All of our manifolds are assumed to be smooth, so we refer to them simply as manifolds.} $\M$ can be thought of as an $n$-dimensional analogue of a surface. Concretely, this is formalized with \textit{charts}, which are smooth bijections $\varphi_x : U_x \to V_x$, where $U_x \subseteq \R^n, x \in V_x \subseteq \M$, that also satisfy a smoothness condition when passing between charts. For charts $\varphi_{x_1}$, $\varphi_{x_2}$ with corresponding $U_{x_1}, V_{x_1}, U_{x_2}, V_{x_2}$, the composed map $\varphi_{x_2}^{-1} \circ \varphi_{x_1} : \varphi_{x_1}^{-1}(V_{x_1} \cap V_{x_2}) \to \varphi_{x_2}^{-1}(V_{x_1} \cap V_{x_2})$ is a diffeomorphism. 


Preexisting manifold normalizing flow works (which we present a complete history of in Section \ref{sec:related}) do not generalize to arbitrary manifolds. Furthermore, many examples present constructions extrinsic to the manifold. In this work, we solve these issues by introducing Manifold Continuous Normalizing Flows (MCNFs), a manifold analogue of Continuous Normalizing Flows. Concretely, we:

\begin{enumerate}[label=(\roman*)]
    \item introduce Neural Manifold ODEs as a generalization of Neural ODEs (seen in Figure \ref{fig:mode}). We leverage existing literature to provide methods for forward mode integration, and we derive a manifold analogue of the adjoint method \cite{pontryagin1962mathematical} for backward mode gradient computation.

    \item develop a dynamic chart method to realize Neural Manifold ODEs in practice. This approach integrates local dynamics in Euclidean space and passes between domains using smooth chart transitions. Because of this, we perform computations efficiently and can accurately compute gradients. Additionally, this allows us to access advanced ODE solvers (without manifold equivalents) and augment the Neural Manifold ODE with existing Neural ODE improvements \cite{Dupont2019AugmentedNO, Grathwohl2019FFJORDFC}.

    \item construct Manifold Continuous Normalizing Flows. These flows are constructed by integrating local dynamics to construct diffeomorphisms, meaning that they are theoretically complete over general manifolds. Empirically, we find that our method outperforms existing manifold normalizing flows on their specific domain.
\end{enumerate}

\section{Related Work}\label{sec:related}

In this section we analyze all major preexisting manifold normalizing flows. Previous methods are, in general, hampered by a lack of generality and are burdensomely constructive.

\textbf{Normalizing Flows on Riemannian Manifolds \cite{Gemici2016NormalizingFO}.} The first manifold normalizing flow work constructs examples on Riemannian manifolds by first projecting onto Euclidean space, applying a predefined Euclidean normalizing flow, and projecting back. Although simple, this construction is theoretically flawed since the initial manifold projection requires the manifold to be diffeomorphic to Euclidean space. This is not always the case, since, for example, the existence of antipodal points on a sphere necessarily implies that the sphere is not diffeomorphic to Euclidean space\footnote{This example is generalized by the notion of conjugate points. Most manifolds have conjugate points and those without are topologically equivalent to Euclidean space.}. As a result, the construction only works on a relatively small and topologically trivial subset of manifolds.

Our work overcomes this problem by integrating local dynamics to construct a global diffeomorphism. By doing so, we do not have to relate our entire manifold with some Euclidean space, but rather only well-behaved local neighborhoods. We test against \cite{Gemici2016NormalizingFO} on hyperbolic space, and our results produce a significant improvement.

\textbf{Latent Variable Modeling with Hyperbolic Normalizing Flows \cite{Bose2020LatentVM}.} In a recent manifold normalizing flow paper, the authors propose two normalizing flows on hyperbolic space---a specific Riemannian manifold. These models, which they name the Tangent Coupling (TC) and Wrapped Hyperboloid Coupling (WHC), are not affected by the aforementioned problem since hyperbolic space is diffeomorphic to Euclidean space. However, various shortcomings exist. First, in our experiments we find that the methods do not seem to conclusively outperform \cite{Gemici2016NormalizingFO}. Second, these methods do not generalize to topologically nontrivial manifolds. This means that these flows produce no additional topological complexity and thus the main benefit of manifold normalizing flows is not realized. Third, the WHC construction is not intrinsic to hyperbolic space since it relies on the hyperboloid equation.

Our method, by contrast, is derived from vector fields, which are natural manifold constructs. This not only allows for generality, but also means that our construction respects the manifold geometry. We compare against \cite{Bose2020LatentVM} on hyperbolic space and find that our results produce a substantial improvement.

\textbf{Normalizing Flows on Tori and Spheres \cite{Rezende2020NormalizingFO}.} In another paper, the authors introduce a variety of normalizing flows for tori and spheres. These manifolds are not diffeomorphic to Euclidean space, hence the authors construct explicit global diffeomorphisms. However, their constructions do not generalize and must be intricately designed with manifold topology in mind. In addition, the primary recursive $\mathbb{S}^n$ flow makes use of non-global diffeomorphisms to the cylinder, so densities are not defined everywhere. The secondary exponential map-based $\mathbb{S}^n$ flow is globally defined but is not computationally tractable for higher dimensions.

In comparison, our work is general, requires only local diffeomorphisms, produces globally defined densities, and is tractable for higher dimensions. We test against \cite{Rezende2020NormalizingFO} on the sphere and attain markedly better results.

\textbf{Other Related Work.} In \cite{Falorsi2019ReparameterizingDO}, the authors define a probability distribution on Lie Groups, a special type of manifold, and as a by-product construct a normalizing flow. The constructed flow is very similar to that found in \cite{Gemici2016NormalizingFO}, but the authors include a $\tanh$ nonlinearity at the end of the Euclidean flow to constrain the input space and guarantee that the map back to the manifold is injective. We do not compare directly against this work since Lie Groups are not general (the $2$-dimensional sphere is not a Lie Group \cite{Stillwell2008NaiveLT}) while Riemannian manifolds are.

There are some related works such as \cite{Wirnsberger2020TargetedFE, Wang2019RiemannianNF, Brehmer2020FlowsFS} that are orthogonal to our considerations as they either (i) develop applications as opposed to theory or (ii) utilize normalizing flows as a tool to study Riemannian metrics.

Concurrent work \cite{mathieu2020Riemannian,Falorsi2020NeuralOD} also investigates the extension of neural ODEs to smoooth manifolds.





\section{Background}

In this section, we present background knowledge to establish naming conventions and intuitively illustrate the constructions used for our work. For a more detailed overview, we recommend consulting a text such as \cite{lee1997riemannian, lee2003introduction, do1992riemannian}.

\subsection{Differential Geometry}



\textbf{Tangent space.} For an $n$-dimensional manifold $\M$, the \textit{tangent space} $T_x\M$ at a point $x \in \M$ is a higher-dimensional analogue of a tangent plane at a point on a surface. It is an $n$-dimensional real vector space for all points $x \in \M$.

For our purposes, tangent spaces will play two roles. First, they provide a notion of derivation which is crucial in defining manifold ODEs. Second, they will oftentimes be used in place of $\R^n$ for our charts (as we map $U_x$ to some open subset of $T_x\M$ through a change of basis from $\R^n \to T_x\M$).

\textbf{Pushforward/Differential.} A derivative (or a \textit{pushforward}) of a function $f : \M \to \N$ between two manifolds is denoted by $D_xf : T_x\M \to T_x\N$. This is a generalization of the classical Euclidean Jacobian (as $\R^n$ is a manifold), and provides a way to relate tangent spaces at different points.

As one might expect, the pushforward is central in the definition of manifold ODEs (analogous to the importance of the common derivative in Euclidean ODEs). We also use it in our dynamic chart method to map tangent vectors of the manifold to tangent vectors of Euclidean space.

\subsection{Riemannian Geometry}

While the above theory is general, to concretize some computational aspects (e.g. how to pick charts) and give background on related manifold normalizing flow work, we define relevant concepts from Riemannian geometry.

\textbf{Riemannian Manifold.} The fundamental object of study in Riemannian geometry is the \textit{Riemannian manifold}. This is a manifold with an additional structure called the \textit{Riemannian metric}, which is a smooth metric $\rho_x : T_x\M \times T_x\M \to \R$ (often denoted as $\inn{\cdot, \cdot}_\rho$). This Riemannian metric allows us to construct a distance on the manifold $d_{\rho} : \M \times \M \to \RR$. Furthermore, any manifold can be given a Riemannian metric.

\textbf{Exponential Map.} The \textit{exponential map} $\exp_x : T_x\M \to \M$ can be thought of as taking a vector $v \in T_x\M$ and following the general direction (on the manifold) such that the distance traveled is the length of the tangent vector. Specifically, the distance on the manifold matches the induced tangent space norm $d_\rho(x, \exp_x(v)) = \norm{v}_\rho := \sqrt{\inn{v, v}_\rho}$. Note that $\exp_x(0) = x$.

The exponential map is crucial in our construction since it acts as a chart. Specifically, if we identify the chart domain with $T_x\M$ then $\exp_x$ is a diffeomorphism when restricted to some local set around~$0$.

\textbf{Special Cases.} Some special cases of Riemannian manifolds include hyperbolic spaces $\mathbb{H}^n = \{x \in \mathbb{R}^{n+1} : -x_1^2 + \sum_{i=2}^{n+1} x_i^2 = -1, \; x_1 > 0\}$, spheres ${\mathbb{S}^n = \{x \in \mathbb{R}^{n+1} : \sum^{n+1}_{i=1} x_i^2 = 1\}}$, and tori $\mathbb{T}^n = (\mathbb{S}^1)^n$. Specific formulas for Riemannian computations are given in Appendix~\ref{appendix:normalizingFlows}. Hyperbolic space is diffeomorphic to Euclidean space, but spheres and tori are not.

\subsection{Manifold Ordinary Differential Equations}

\textbf{Manifold ODE.} Finally, we introduce the key objects of study: manifold ordinary differential equations. A manifold ODE is an equation which relates a curve $\mbf{z} : [t_s, t_e] \to \M$ to a vector field $f$ and takes the form
\begin{equation}\label{eqn:manifoldOde}
    \frac{d\mbf{z}(t)}{dt} = f(\mbf{z}(t), t) \in T_{\mbf{z}(t)} \M \qquad \mbf{z}(t_s) = z_s
\end{equation}
$\mbf{z}$ is a \textit{solution} to the ODE if it satisfies Equation \ref{eqn:manifoldOde} with initial condition $z_s$. Similarly to the case of classical ODEs, local solutions are guaranteed to exist under sufficient conditions on $f$ \cite{Hairer2011manifold}.

\section{Neural Manifold Ordinary Differential Equations}

To leverage manifold ODEs in a deep learning framework similar to Neural ODEs \cite{Chen2018NeuralOD}, we parameterize the dynamics $f$ by a neural network with parameters $\theta$. We define both forward pass and backward pass gradient computation strategies in this framework. Unlike Neural ODEs, forward and backward computations do not perfectly mirror each other since forward mode computation requires explicit manifold methods, while the backward can be defined solely through an ODE in Euclidean space.

\subsection{Forward Mode Integration}\label{sec:forward}

The first step in defining our Neural Manifold ODE block is the forward mode integration. We review and select appropriate solvers from the literature; for a more thorough introduction we recommend consulting a text such as \cite{Crouch1993NumericalMODE, Hairer2011manifold}. Broadly speaking, these solvers can be classified into two groups:

\begin{enumerate}[label=(\roman*)]
    \item \textit{projection methods} that embed the manifold into Euclidean space $\R^d$, integrate with some base Euclidean solver, and project to the manifold after each step. Projection methods require additional manifold structure; in particular, $\M$ must be the level set of some smooth function $g: \R^d \to \R$.

    \item \textit{implicit methods} that solve the manifold ODE locally using charts. These methods only require manifold-implicit constructions.
\end{enumerate}

Projection methods are conceptually simple, but ultimately suffer from generality issues. In particular, for manifolds such as the open ball or the upper half-space, there is no  principled way to project off-manifold points back on. Furthermore, even in nominally well-defined cases such as the sphere, there may still exist points such as the origin for which the projection is not well-defined. 

Implicit methods, by contrast, can be applied to any manifold and do not require a level set representation. Thus, this approach is more amenable to our generality concerns (especially since we wish to work with, for example, hyperbolic space). However, difficulty in defining charts restricts applicability. Due to this reason, implicit schemes often employ additional structure to define manifold variations of step-based solvers \cite{BieleckiRiemannianNormal, Crouch1993NumericalMODE}. For example, on a Riemannian manifold one can define a variant of an Euler Method solver with update step $z_{t + \epsilon} = \exp_{z_t}(\epsilon f(z_t, t))$ using the Riemannian exponential map. On a Lie group, there are more advanced Runge-Kutta solvers that use the Lie Exponential map and coordinate frames \cite{Crouch1993NumericalMODE, Hairer2011manifold}.

\subsection{Backward Mode Adjoint Gradient Computation}\label{sec:backward}

In order to fully incorporate manifold ODEs into the deep learning framework, we must also efficiently compute gradients. Similar to \cite{Chen2018NeuralOD}, we develop an adjoint method to analytically calculate the derivative of a manifold ODE instead of directly differentiating through the solver. Similar to manifold adjoint methods for partial differential equations, we use an ambient space \cite{Zahr2016AFD}.

\begin{thm}\label{thm:manifoldAdjoint}
    Suppose we have some manifold ODE as given in Equation \ref{eqn:manifoldOde} and we define some loss function $L: \M \to \R$. Suppose that there is an embedding of $\M$ in some Euclidean space $\R^d$ and we identify $T_x\M$ with an $n$-dimensional subspace of $\R^d$. If we define the adjoint state to be $\mbf{a}(t) = D_{\mbf{z}(t)} L$, then the adjoint satisfies
    \begin{equation}\label{eqn:manifoldAdjoint}
        \frac{d\mbf{a}(t)}{dt} = - \mbf{a}(t) D_{\mbf{z}(t)}f(\mbf{z}(t), t)
    \end{equation}
\end{thm}

\begin{remark}
    This theorem resembles the adjoint method in \cite{Chen2018NeuralOD} precisely because of our ambient space condition. In particular, our curve $\mbf{z}: [t_s, t_e] \to \M$ can be considered as a curve in the ambient space. Furthermore, we do not lose any generality since such an embedding always exists by the Whitney Embedding Theorem \cite{Whitney1935Differentiable}, if we simply set $d = 2n$.
\end{remark}

 The full proof of Theorem \ref{thm:manifoldAdjoint} can be found in Appendix \ref{thm:manifoldAdjointAppendix}. Through this adjoint state, gradients can be derived for other parameters in the equation such as $t_s, t_e$, initial condition $z_s$, and weights $\theta$.

\begin{figure}[htb!]
\centering
 \includegraphics[width=0.8\textwidth]{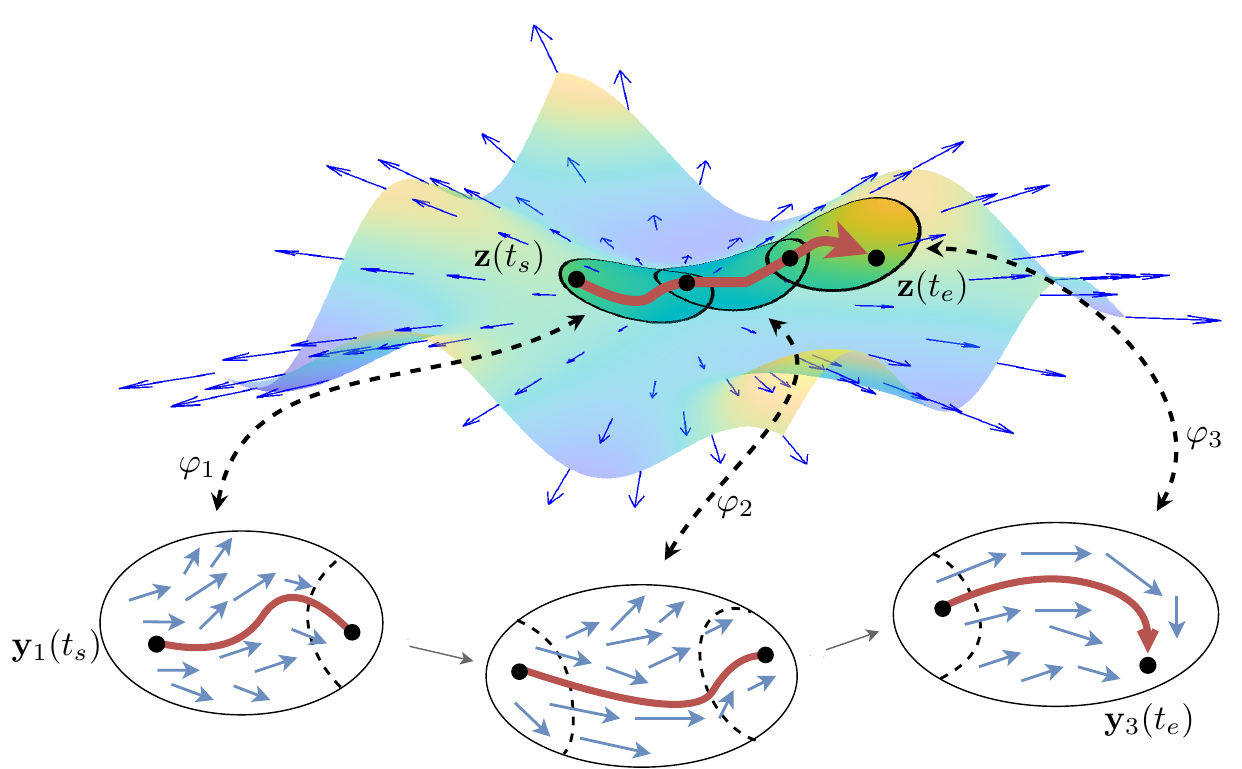}
 \caption{Solving a manifold ODE with our dynamic chart method. We use $3$ charts.}
 \label{fig:dynamicChart}
\end{figure}

\section{Dynamic Chart Method}\label{sec:trivialization}

While our above theory is fully expressive and general, in this section we address certain computational issues and augment applicability with our \textit{dynamic chart method}.


\begin{algorithm}[t]
    \SetAlgoLined
    Given $f$, local charts $\varphi_x$, starting condition $z_s$ and starting/ending times of $t_s, t_e$.\\
    Initialize $z \leftarrow z_s$, $\tau \leftarrow t_s$\\
    \While{$\tau < t_e$}{
        Construct an equivalent differential equation $\frac{d\mbf{y}(t)}{dt} = D_{\varphi_z(\mbf{y}(t))} \varphi_z^{-1} \circ f(\varphi_z(\mbf{y}(t)), t)$ with initial condition $\mbf{y}(\tau) = \varphi_z^{-1}(z)$\\
        Solve $\mbf{y}$ locally using some numerical integration technique in Euclidean space. Specifically, solve in some interval $[\tau, \tau + \epsilon]$ for which $\mbf{z}([\tau, \tau + \epsilon]) \subseteq \mrm{im} \varphi_z$.\\
        $z \leftarrow \varphi_z(\mbf{y}(\tau + \epsilon)), \tau \leftarrow \tau + \epsilon$
    }
    \caption{Dynamic Chart Forward Pass}
    \label{alg:implicitNumerical}
\end{algorithm}

The motivation for our dynamic chart method comes from \cite{LezcanoCasado2019TrivializationsFG}, where the author introduces a \textit{dynamic manifold trivialization} technique for Riemannian gradient descent. Here, instead of applying the traditional Riemannian gradient update $z_{t + 1} = \exp_{z_t}(-\eta\nabla_{z_t} f)$ for $N$ time steps, the author repeatedly applies $n \ll N$ local updates. Each update consists of a local diffeomorphism to Euclidean space, $n$~equivalent Euclidean gradient updates, and a map back to the manifold. This allows us to lower the number of expensive exponential map calls and invoke existing Euclidean gradient optimizers such as Adam \cite{Kingma2014adam}. This is similar in spirit to \cite{Gemici2016NormalizingFO}, but crucially only relies on local diffeomorphisms rather than a global diffeomorphism.

We can adopt this strategy in our Neural Manifold ODE. Specifically, we develop a generalization of the dynamic manifold trivialization for the manifold ODE forward pass. In a similar spirit, we use a local chart to map to Euclidean space, solve an equivalent Euclidean ODE locally, and project back to the manifold using the chart. The full forward pass is given by Algorithm \ref{alg:implicitNumerical} and is visualized Figure~\ref{fig:dynamicChart}.

We present two propositions which highlight that this algorithm is guaranteed to converge to the manifold ODE solution. The first shows that $\varphi_z(\mbf{y})$ solves the manifold differential equation locally and the second shows that we can pick a finite collection of charts such that we can integrate to time~$t_e$.

\begin{prop}[Correctness]
    If $\mbf{y}(t) : [\tau,\tau+\epsilon] \to \R^n$ is a solution to $\frac{d\mbf{y}(t)}{dt} = D_{\varphi_{z}(\mbf{y}(t))} \varphi_{z}^{-1} \circ f(\varphi_{z}(\mbf{y}(t)), t)$ with initial condition $\mbf{y}(\tau) = \varphi_{z}^{-1}(z)$, then $\mbf{z}(t) = \varphi_{z}(\mbf{y}(t))$ is a valid solution to Equation \ref{eqn:manifoldOde} on $[\tau,\tau+\epsilon]$.
\end{prop}

\begin{prop}[Convergence]
    There exists a finite collection of charts $\{\varphi_i\}_{i = 1}^k$ s.t. $\mbf{z}([t_s, t_e]) \subseteq \displaystyle\bigcup_{i = 1}^k \mrm{im} \varphi_i$.
\end{prop}

Proofs of these propositions are given in Appendix~\ref{appendix:trivialization}. Note that this forward integration is implicit, indicating the connections between \cite{hairer2010solving, Hairer2011manifold} and \cite{LezcanoCasado2019TrivializationsFG}. We realize this construction by finding a principled way to pick charts and incorporate this method into neural networks by developing a backward gradient computation.

We can intuitively visualize this dynamic chart forward pass as a sequence of $k$ Neural ODE solvers $\{\mathrm{ODE}_i\}_{i \in [k]}$ with chart transitions $\varphi_{i_2}^{-1} \circ \varphi_{i_1}$ connecting them. Here, we see how the smooth chart transition property comes into play, as passing between charts is the necessary step in constructing a global solution. In addition, this construction provides a \textit{chart-based backpropagation}. Under this dynamic chart forward pass, we can view a Neural Manifold ODE block as the following composition of Neural ODE blocks and chart transitions:
\begin{equation}\label{eqn:MODEComposition}
    \mathrm{MODE} = \varphi_k \circ \mathrm{ODE}_k \circ (\varphi_k^{-1} \circ \varphi_{k - 1}) \circ \dots \circ (\varphi_2^{-1} \circ \varphi_1) \circ \mathrm{ODE}_1 \circ \varphi_1^{-1}
\end{equation}
This allows for gradient computation through backpropagation. We may differentiate through the Neural ODE blocks by the Euclidean adjoint method \cite{pontryagin1962mathematical,Chen2018NeuralOD}.

To finalize this method, we give a strategy for picking these charts for Riemannian manifolds\footnote{Note that we do not lose generality since all manifolds can be given a Riemannian metric.}. As previously mentioned, the exponential map serves as a local diffeomorphism from the tangent space (which can be identified with $\R^n$) and the manifold, so it acts as a chart. Similar to \cite{Falorsi2019ReparameterizingDO}, at each point there exists a radius $r_x$ such that $\exp_x$ is a diffeomorphism when restricted to $B_{r_x} := \{v \in T_x\M : \norm{v}_\rho < r_x\}$. With this information, we can solve the equivalent ODE with solution $\mbf{y}$ until we reach a point $y > (1 - \epsilon) \norm{r_x}_\rho$, at which point we switch to the exponential map chart defined around $\exp_x(y)$. Complete details are provided in Appendix~\ref{appendix:dynamicChart}.

Our dynamic chart method is a significant advancement over previous implicit methods since we can easily construct charts as we integrate. Furthermore, it provides many practical improvements over vanilla Neural Manifold ODEs. Specifically, we

\begin{enumerate}[label=(\roman*)]
    \item \textit{can perform faster evaluations.} The aforementioned single step algorithms rely on repeated use of the Lie and Riemannian exponential maps. These are expensive to compute and our method can sidestep this expensive evaluation. In particular, the cost is shifted to the derivative of the chart, but by defining dynamics $g$ on the tangent space directly, we can avoid this computation. We use this for our hyperbolic space construction, where we simply solve $\frac{d\mbf{y}(t)}{dt} = g(\mbf{y}(t), t)$.

    \item \textit{avoid catastrophic gradient instability.} If the dimension of $\M$ is less than the dimension of the ambient space, then the tangent spaces are of measure $0$. This means that the induced error from the ODE solver will cause our gradients to leave their domain, resulting in a catastrophic failure. However, since the errors in Neural ODE blocks do not cause the gradient to leave their domain and as our charts induce only a precision error, our dynamic chart method avoids this trap.

    \item \textit{access a wider array of ODE advancements.} While substantial work has been done for manifold ODE solvers, the vast majority of cutting edge ODE solvers are still restricted to Euclidean space. Our dynamic chart method can make full use of these advanced solvers in the Neural ODE blocks. Additionally, Neural ODE improvements such as \cite{Dupont2019AugmentedNO, Jia2019Neural} can be directly integrated without additional manifold constructions.
\end{enumerate}

\section{Manifold Continuous Normalizing Flows} 

With our dynamic chart Neural Manifold ODE, we can construct a Manifold Continuous Normalizing Flow (MCNF). The value can be integrated directly through the ODE, so all that needs to be given is the change in log probability. Here, we can invoke continuous change of variables \cite{Chen2018NeuralOD} on the Neural ODE block and can use the smooth chart transition property (which guarantees that the charts are diffeomorphisms) to calculate the final change in probability as:
\begin{align}\label{eqn:mcnflog}
\log p(\mrm{MODE}) = \log \pi &-\sum_{i = 1}^k \paren{\log\det  |D \varphi_i| + \log \det |D \varphi_i^{-1}| + \int \mrm{tr}(D \varphi_i^{-1} \circ f)}
\end{align}
Note that we drop derivative and integration arguments for brevity. A full expression is given in Appendix~\ref{appendix:normalizingFlows}.

For our purposes, the last required computation is the determinant of the $D_v\exp_x$. We find that these can be evaluated analytically, as shown in the cases of the hyperboloid and sphere \cite{Nagano2019AWN, Skopek2019MixedcurvatureVA}.

Since the MCNF requires only the local dynamics (which are in practice parameterized by the exponential map), this means that the construction \textit{generalizes to arbitrary manifolds}. Furthermore, we \textit{avoid diffeomorphism issues}, such as in the case of two antipodal points on the sphere, simply by restricting our chart domains to never include these conjugate points.

\begin{wrapfigure}{R}{0.6\textwidth}
  \vspace{-20pt}
  \newcommand{\headsize}{\scriptsize}
  \centering
  \begin{tabular}{cccc}
   \headsize{Target} & \headsize{WHC \cite{Bose2020LatentVM}} & \headsize{PRNVP \cite{Gemici2016NormalizingFO}} & \headsize{MCNF (Ours)}\\
   \midrule[1pt]
  \includegraphics[width=0.09\textwidth]{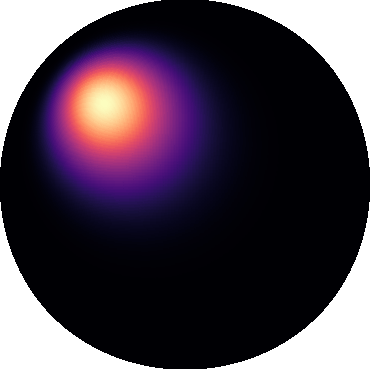} &
  \includegraphics[width=0.09\textwidth]{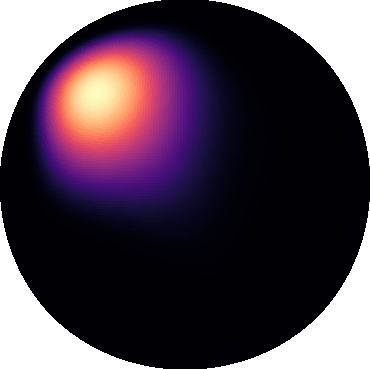} &
  \includegraphics[width=0.09\textwidth]{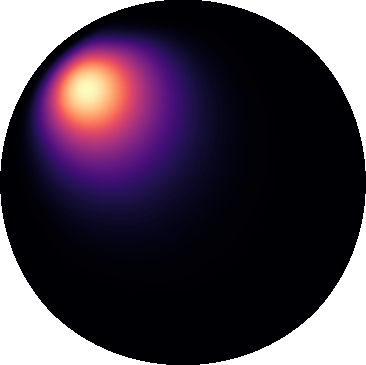} &
  \includegraphics[width=0.09\textwidth]{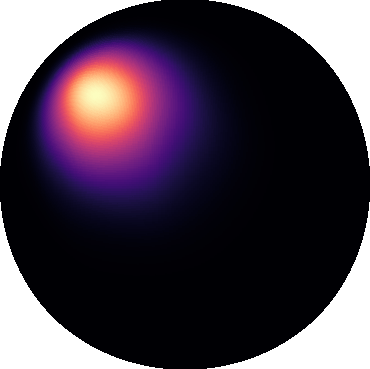}\\
  
  \includegraphics[width=0.09\textwidth]{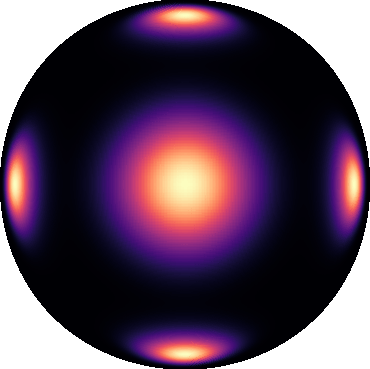} &
  \includegraphics[width=0.09\textwidth]{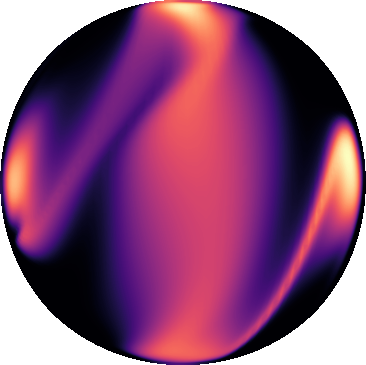} &
  \includegraphics[width=0.09\textwidth]{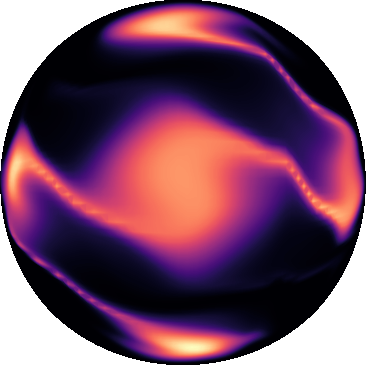} &
  \includegraphics[width=0.09\textwidth]{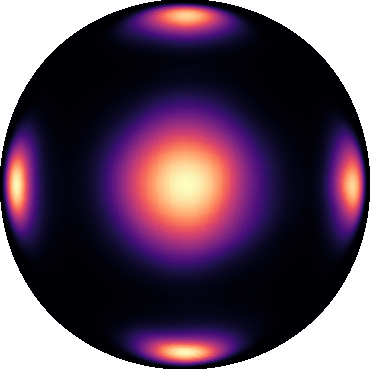} \\

  \includegraphics[width=0.09\textwidth]{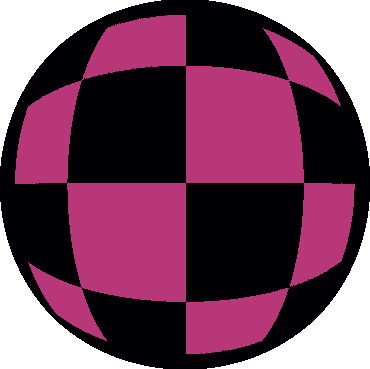} &
  \includegraphics[width=0.09\textwidth]{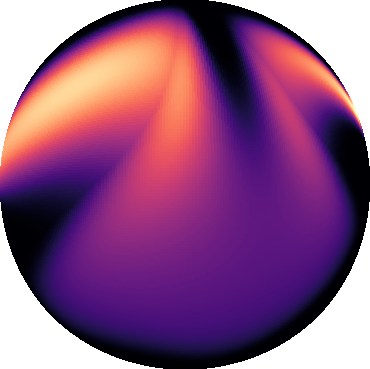} &
  \includegraphics[width=0.09\textwidth]{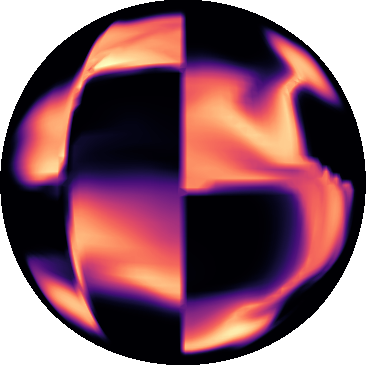} &
  \includegraphics[width=0.09\textwidth]{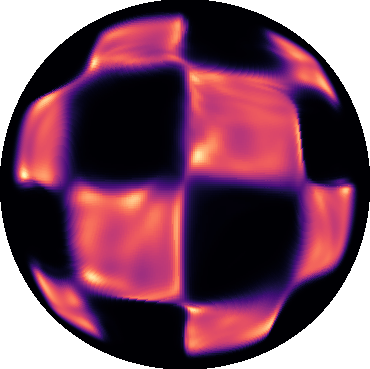} \\
  
  \includegraphics[width=0.09\textwidth]{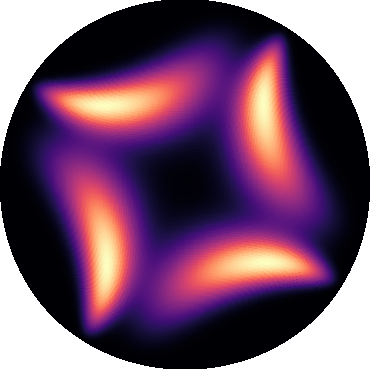} &
  \includegraphics[width=0.09\textwidth]{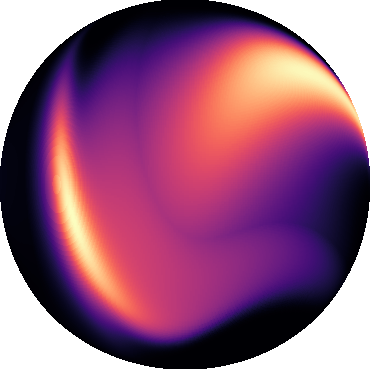} &
  \includegraphics[width=0.09\textwidth]{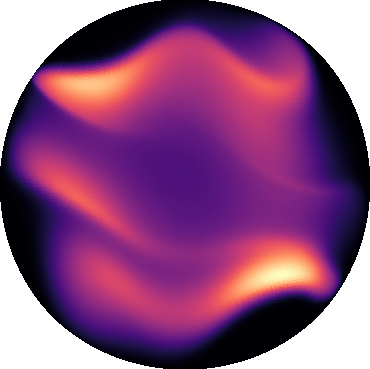} &
  \includegraphics[width=0.09\textwidth]{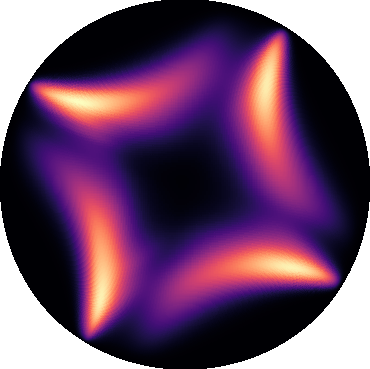}
  \end{tabular}
  
  \caption{Density estimation on the hyperboloid $\mathbb{H}^2$, which is projected to the Poincar\'e Ball for visualization.}
  \label{fig:hyperbolicDensity}
  
    \begin{tabular}{ccc}
         \headsize{Target} & \headsize{NCPS \cite{Rezende2020NormalizingFO}} & \headsize{MCNF (Ours)} \\
         \midrule[1pt]
         \includegraphics[width=0.14\textwidth]{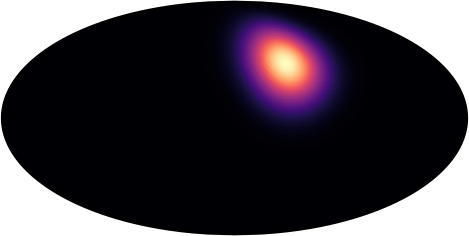} &
         \includegraphics[width=0.14\textwidth]{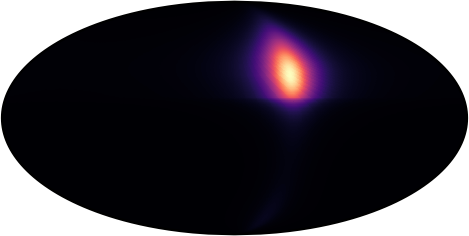}  &
         \includegraphics[width=0.14\textwidth]{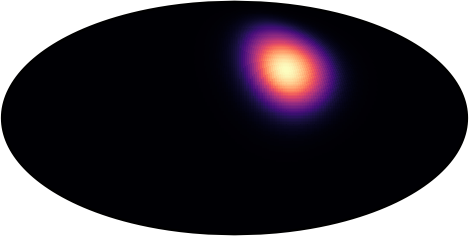}  \\
         
         \includegraphics[width=0.14\textwidth]{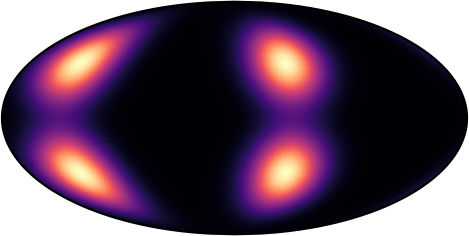} &
         \includegraphics[width=0.14\textwidth]{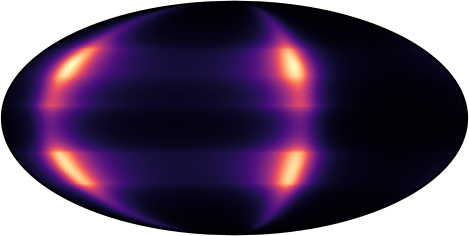} &
         \includegraphics[width=0.14\textwidth]{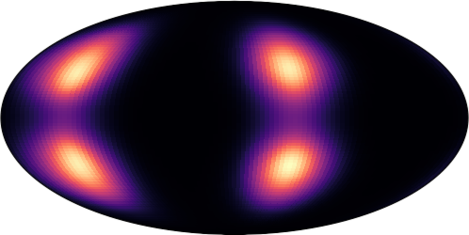} \\
         
         \includegraphics[width=0.14\textwidth]{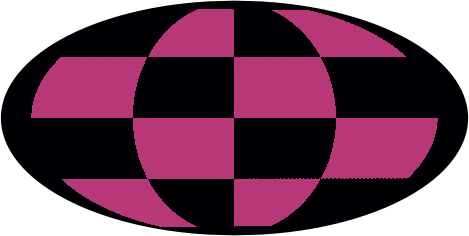} &
         \includegraphics[width=0.14\textwidth]{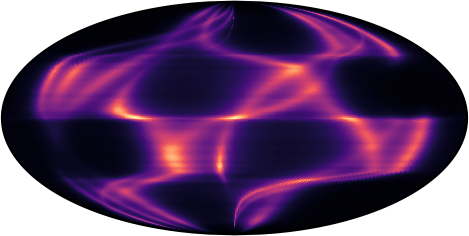} &
         \includegraphics[width=0.14\textwidth]{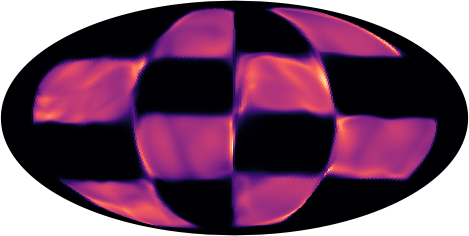} \\
    \end{tabular}
    \caption{Density estimation on the sphere $\mathbb{S}^2$, which is projected to two dimensions by the Mollweide projection.}
    \label{fig:sphereDensity}
    \vspace{-25pt}
\end{wrapfigure}

\section{Experiments}

To test our MCNF models, we run density estimation and variational inference experiments. Though our method is general, we take $\M$ to be two spaces of particular interest: hyperbolic space $\mathbb H^n$ and the sphere $\mathbb S^n$. Appendix~\ref{appendix:normalizingFlows} concretely details the computation of MCNF in these spaces. Full experimental details can be found in Appendix~\ref{appendix:trainingDetails}.

\subsection{Density Estimation}\label{subsec:DensityEstimation}

We train normalizing flows for estimation of densities in the hyperbolic space $\mathbb H^2$ and the sphere $\mathbb S^2$, as these spaces induce efficient computations and are easy to visualize. For hyperbolic space, the baselines are Wrapped Hyperboloid Coupling (WHC) \cite{Bose2020LatentVM} and Projected NVP (PRNVP), which learns RealNVP over the projection of the hyperboloid to Euclidean space \cite{Gemici2016NormalizingFO,Dinh2017DensityEU}. On the sphere, we compare with the recursive construction of \cite{Rezende2020NormalizingFO}, with non-compact projection used for the $\mathbb S^1$ flow (NCPS). As visualized in Figures~\ref{fig:hyperbolicDensity} and \ref{fig:sphereDensity}, our normalizing flows are able to match complex target densities with significant improvement over the baselines. MCNF is even able to fit discontinuous and multi-modal target densities; baseline methods cannot fit these cases as they struggle with reducing probability mass in areas of low target density.

\subsection{Variational Inference}

We train a hyperbolic VAE \cite{Mathieu2019HierarchicalRW} and Euclidean VAE \cite{Kingma2014AutoEncodingVB} for variational inference on Binarized Omniglot \cite{Lake2015Human} and Binarized MNIST \cite{Lecun98gradient-basedlearning}. Both of these datasets have been found to empirically benefit from hyperbolic embeddings in prior work \cite{Nagano2019AWN,Mathieu2019HierarchicalRW,Bose2020LatentVM,Khrulkov2019HyperbolicIE,Skopek2019MixedcurvatureVA}. We compare different flow layers in the latent space of the VAEs. For the Euclidean VAE, we compare with RealNVP \cite{Dinh2017DensityEU} and CNF \cite{Chen2018NeuralOD,Grathwohl2019FFJORDFC}. Along with the two hyperbolic baselines used for density estimation, we also compare against the Tangent Coupling (TC) model in \cite{Bose2020LatentVM}. As shown in Table~\ref{tbl:TestLL}, our continuous flow regime is more expressive and learns better than all hyperbolic baselines. In low dimensions, along with the other hyperbolic models, our approach tends to outperform Euclidean models on Omniglot and MNIST. However, in high dimension the hyperbolic models do not reap as much benefit; even the baseline HVAE does not consistently outperform the Euclidean VAE.

\begin{table}[t]
\caption{MNIST and Omniglot average negative test log likelihood (lower is better) and standard deviation over five trials for varying dimensions.}
\label{tbl:TestLL}
\centering
\scriptsize
\begin{tabular}{clllllll}
\cmidrule[\heavyrulewidth]{2-8}
  & & \multicolumn{3}{c}{MNIST} & \multicolumn{3}{c}{Omniglot} \\
  \cmidrule(lr){3-5} \cmidrule(lr){6-8}
  & & \tabcent{2}  & \tabcent{4}  & \tabcent {6} & \tabcent{2} & \tabcent{4}  & \tabcent{6} \\
\cmidrule[\heavyrulewidth]{2-8}
\addlinespace[3pt]
\multirow{3}{*}{\rotatebox[origin=c]{90}{Euclidean}} & VAE\cite{Kingma2014AutoEncodingVB} & $143.06\pm .3$ & $117.57 \pm .5$ & $102.13 \pm .2$ & $154.31 \pm .5$ & $143.37 \pm .2$ & $138.65 \pm .1$ \\
& RealNVP \cite{Dinh2017DensityEU}   & $142.09 \pm .7$ & $116.32 \pm .7$ & $100.95 \pm .1$ & $153.93\pm .5$   & $142.98 \pm .3$  & $\mbf{137.21} \pm .1$  \\
& CNF \cite{Grathwohl2019FFJORDFC}   & $141.16 \pm .4$ & $116.28 \pm .5$& $100.64 \pm .1$ & $154.05 \pm .2$ & $143.11 \pm .4$  & $137.32 \pm .7$  \\
\addlinespace[3pt]
\cmidrule[.3pt]{2-8}
\addlinespace[3pt]
\multirow{4}{*}{\rotatebox[origin=c]{90}{Hyperbolic}} & HVAE\cite{Mathieu2019HierarchicalRW} & $140.04\pm .9$ & $114.81 \pm .8$ & $100.45 \pm .2$ & $153.97 \pm .3$ & $144.10 \pm .8$ & $138.02 \pm .3$ \\
& TC \cite{Bose2020LatentVM}  & $139.58 \pm .4$ & $114.16 \pm .6$ & $100.45 \pm .2$ & $157.11 \pm 2.7$ & $143.05 \pm .5$ & $137.49 \pm .1$ \\
& WHC \cite{Bose2020LatentVM}  & $140.46 \pm 1.3$ & $113.78 \pm .3$ & $100.23 \pm .1$ & $158.08 \pm 1.0$ & $143.23 \pm .6$ & $137.64 \pm .1$ \\
& PRNVP \cite{Gemici2016NormalizingFO} & $140.43 \pm 1.8$ & $113.93 \pm .3$ & $100.06 \pm .1$ & $156.71 \pm 1.7$ & $143.00 \pm .3$ & $137.57 \pm .1$  \\
& MCNF (ours)       & $\mbf{138.14} \pm .5$ & $\mbf{113.47} \tiny{\pm .3}$ & $\mathbf{99.89} \pm .1$ & $\mbf{152.98}\pm .1$ & $142.99 \pm .3$ & $137.29 \pm .1$ \\
\addlinespace[3pt]
\cmidrule[\heavyrulewidth]{2-8}
\end{tabular}
\end{table}

\vspace{-5pt}
\section{Conclusion}
\vspace{-5pt}

We have presented Neural Manifold ODEs, which allow for the construction of continuous time manifold neural networks. In particular, we introduce the relevant theory for defining ``pure'' Neural Manifold ODEs and augment it with our dynamic chart method. This resolves numerical and computational cost issues while allowing for better integration with modern Neural ODE theory. With this framework of continuous manifold dynamics, we develop Manifold Continuous Normalizing Flows. Empirical evaluation of our flows shows that they outperform existing manifold normalizing flow baselines on density estimation and variational inference tasks. Most importantly, our method is completely general as it does not require anything beyond local manifold structure.

We hope that our work paves the way for future development of manifold-based deep learning. In particular, we anticipate that our general framework will lead to other continuous generalizations of manifold neural networks. Furthermore, we expect to see application of our Manifold Continuous Normalizing Flows for topologically nontrivial data in lattice quantum field theory, motion estimation, and protein structure prediction.

\section{Broader Impact}

As mentioned in the introduction, our method has applications to physics, robotics, and biology.  While there are ethical and social concerns with parts of these fields, our work is too theoretical for us to say for sure what the final impact will be. For deep generative models, there are overarching concerns with generating fake data for malicious use (e.g. deepfake impersonations). However, our work is more concerned with accurately modelling data topology rather than generating hyper-realistic vision or audio samples, so we do not expect there to be any negative consequence in this area.

 \section{Acknowledgements}

 We would like to acknowledge Prof. Austin Benson and Junteng Jia for their insightful comments and suggestions. In addition, we would like to thank Facebook AI for funding equipment that made this work possible. We would also like to thank Joey Bose for providing access to his prerelease code.

\bibliography{nmode}
\bibliographystyle{plain}

\newpage

\appendix

\section{Proofs of Propositions}

\subsection{Dynamic Chart Method}\label{appendix:trivialization}

\begin{prop}[Correctness]
    If $\mbf{y}(t) : [\tau,\tau+\epsilon] \to \R^n$ is a solution to $\frac{d\mbf{y}(t)}{dt} = D_{\varphi_{z}(\mbf{y}(t))} \varphi_{z}^{-1} \circ f(\varphi_{z}(\mbf{y}(t)), t)$ with initial condition $\mbf{y}(\tau) = \varphi_{z}^{-1}(z)$, then $\mbf{z}(t) = \varphi_{z}(\mbf{y}(t))$ is a valid solution to Equation \ref{eqn:manifoldOde} on $[\tau,\tau+\epsilon]$.
\end{prop}
\begin{proof}
    We see that if $\mbf{z}(t) = \varphi_z(\mbf{y}(t))$ then for all $t' \in [\tau, \tau + \epsilon]$
    
    \begin{align}
        \frac{d\mbf{z}(t')}{dt} &= D_{\mbf{y}(t')} \varphi_z \circ \frac{d\mbf{y}(t')}{dt}\\
        &= D_{\mbf{y}(t')}\varphi_z \circ D_{\varphi_z(\mbf{y}(t'))} \varphi_z^{-1} \circ f(\varphi_z(\mbf{y}(t')), t')\\
        &= f(\varphi_z(\mbf{y}(t')), t')\\
        &= f(\mbf{z}(t'), t')
    \end{align}
\end{proof}

\begin{prop}[Convergence]
    There exists a finite collection of charts $\{\varphi_i\}_{i = 1}^k$ s.t. $\mbf{z}([t_s, t_e]) \subseteq \displaystyle\bigcup_{i = 1}^k \mrm{im} \varphi_i$.
\end{prop}
\begin{proof}
    Around each point $z \in \mbf{z}([t_s, t_e])$ pick a chart $\varphi_z$. Then note that $\{\varphi_z\}_{z \in \mbf{z}([t_s, t_e])}$ satisfies $\mbf{z}([t_s, t_e]) \subseteq \displaystyle\bigcup_z \mrm{im} \varphi_z$. But, the curve is compact since $[t_s, t_e]$ is compact (and $\mbf{z}$ is assumed to be continuous) so we can take a finite subset $\{\varphi_i\}_{i = 1}^k$ that covers the curve.
\end{proof}

\subsection{Derivation of Gradient for Adjoint State}

In this section we prove Theorem \ref{thm:manifoldAdjoint}. The proof follows from the analogous one in \cite{Chen2018NeuralOD}, though we replace certain operations with their manifold counterparts.

\begin{thm}\label{thm:manifoldAdjointAppendix}
    Suppose we have some manifold ODE as given in Equation \ref{eqn:manifoldOde} and we define some loss function $L: \M \to \R$. Suppose that there is an embedding of $\M$ in some Euclidean space $\R^d$ and we identify $T_x\M$ with an $n$-dimensional subspace of $\R^d$. If we define the adjoint state to be $\mbf{a}(t) = D_{\mbf{z}(t)} L$, then the adjoint satisfies
    \begin{equation}
        \frac{d\mbf{a}(t)}{dt} = - \mbf{a}(t) D_{\mbf{z}(t)}f(\mbf{z}(t), t)
    \end{equation}
\end{thm}
\begin{proof}
    Consider the first order approximation of $\mbf{z}(t + \epsilon)$. Since we are embedding $T_x\M \subseteq \R^d$, then under standard $\R^d$ computations we have that
    
    \begin{equation}\label{eqn:firstOrder}
        \mbf{z}(t + \epsilon) = \mbf{z}(t) + \epsilon f(\mbf{z}(t), t) + \mathcal{O}(\epsilon^2)
    \end{equation}
    
    We set $T_\epsilon(\mbf{z}(t), t) := \mbf{z}(t + \epsilon)$. As in the original adjoint method derivation \cite{Chen2018NeuralOD}, we utilize the chain rule
    
    \begin{equation}\label{eqn:AdjointChainRule}
        D_{\mbf{z}(t)} L =D_{\mbf{z}(t + \epsilon)} L \circ D_{\mbf{z}(t)} T_\epsilon(\mbf{z}(t), t) \text{\quad or \quad} \mbf{a}(t)=\mbf{a}(t+\epsilon)D_{\mbf{z}(t)} T_\epsilon(\mbf{z}(t), t)
    \end{equation}
    
    Using these, we get that
    \begin{align}
        \frac{d\mbf{a}(t)}{dt} &=\lim\limits_{\epsilon\rightarrow 0+}\frac{\mbf{a}(t+\epsilon)-\mbf{a}(t)}{\epsilon} & \\
        &=\lim\limits_{\epsilon\rightarrow 0+}\frac{\mbf{a}(t+\epsilon)-\mbf{a}(t+\epsilon)D_{\mbf{z}(t)} T_\epsilon(\mbf{z}(t), t)}{\epsilon} &\text{(by Equation \ref{eqn:AdjointChainRule})} \\
        &=\lim\limits_{\epsilon\rightarrow 0+}\frac{\mbf{a}(t+\epsilon)-\mbf{a}(t+\epsilon)D_{\mbf{z}(t)}(\mbf{z}(t)+\epsilon f(\mbf{z}(t), t) + \mathcal{O}(\epsilon^{2}))}{\epsilon} &\text{(by Equation \ref{eqn:firstOrder})} \\
        &=\lim\limits_{\epsilon\rightarrow 0+}\frac{-\epsilon \mbf{a}(t+\epsilon)D_{\mbf{z}(t)} f(\mbf{z}(t), t) +O(\epsilon^{2})}{\epsilon} & \\
        &= -\mbf{a}(t)D_{\mbf{z}(t)}f(\mbf{z}(t), t)
    \end{align}
\end{proof}

\section{Computation of Manifold Continuous Normalizing Flows}\label{appendix:normalizingFlows}

\subsection{Background}

In Euclidean space, a normalizing flow $f$ is a diffeomorphism $f: \RR^n \to \RR^n$ that maps a base probability distribution into a more complex probability distribution. Suppose $z \sim \pi$ is a sample from the simple distribution. By the change of variables formula, the target density of an $x$ in terms of $p$ (the complex distribution) is
\begin{equation}
\log p(x) = \log \pi(z) - \log \det \abs{\parderiv{f^{-1}}{z}} 
\end{equation}
There exist a variety of functions $f$ which constrain the log Jacobian determinant to be computationally tractable. An important one is the Continuous Normalizing Flow (CNF). CNFs construct $f$ to be the solution of an ODE \cite{Chen2018NeuralOD,Grathwohl2019FFJORDFC}. Explicitly, let $t_0, t_1$ be starting and ending times with $t_0 < t_1$, and consider the ordinary differential equation $\frac{d\mbf{z}(t)}{dt} = g(\mbf{z}(t), t; \theta)$, where $\theta$ parameterizes the dynamics $g$. For a sample from the base distribution $z \sim \pi$, solving this ODE with initial condition $\mbf{z}(t_0) = z$ gives the sample from the target distribution $x = \mbf{z}(t_1)$. The change in the log density given by this model satisfies an ordinary differential equation called the instantaneous change of variables formula:
\begin{equation}\label{eqn:instcov}
    \frac{d \log p(\mbf{z}(t))}{dt} = -\mrm{tr}\left(D_{\mbf{z}(t)} g\right)
\end{equation}
We can therefore quantify the final probability as
\begin{equation}
    \log p(\mbf{z}(t_1)) = \log p(\mbf{z}(t_0)) - \int_{t_0}^{t_1} \mrm{tr}\left(D_{\mbf{z}(t)} g \right) \; dt.
\end{equation}

\subsection{Manifold Continuous Normalizing Flows}

For our MCNF we split up the original time $[t_s, t_e]$ into $[t_i, t_{i + 1}]$ for $i \in [k]$ where $t_s = t_1 < t_2 < \dots < t_k < t_{k + 1} = t_e$. From our curve $\mbf{z}$ we can select $z_i = \mbf{z}(t_i)$, and we have charts $\varphi_i : U_i \to V_i$ s.t. $z_i, z_{i + 1} \in V_i$. If our dynamics are determined by $\frac{d\mbf{z}(t)}{dt} = f(\mbf{z}(t), t; \theta)$ then this locally takes the form $\varphi_i(\widehat{f_i}(\varphi_i^{-1}(\mbf{z}(t), t; \theta)))$, in which $\widehat{f_i}$ is a CNF. The update after passing through a chart $\varphi_i$ and integrating is given by

\begin{equation*}
    \log p(z_{i+1}) = \log \pi(z_i) - \paren{\log \abs{ \det  D_{\widehat{f_i}(\varphi^{-1}(z_i))} \varphi_i} + \int_{t_i}^{t_{i + 1}} \mrm{tr} (D_{\varphi_i^{-1}(z_i)} \widehat{f_i})dt + \log \abs{\det D_{z_i} \varphi_i^{-1}}}
\end{equation*}

where $\log \abs{\det D \varphi_i}$ is a shorthand for the Riemannian probability update induced by the chart. Note that in general this is not the determinant (for instance when the map goes from elements in $\R^n$ to $\R^d$ where $d > n$). In practice it ends up being quite similar.

We can consider our manifold ODE as a composition of these updates. Therefore, we have that

\begin{align}
    \log &p(f(z)) = \nonumber \\
    &\log \pi(z) - \sum_{i = 1}^k \paren{\log \abs{ \det  D_{\widehat{f_i}(\varphi^{-1}(z_i))} \varphi_i} + \int_{t_i}^{t_{i + 1}} \mrm{tr} (D_{\varphi_i^{-1}(z_i)} \widehat{f_i})dt + \log \abs{\det D_{z_i} \varphi_i^{-1}}}
\end{align}

For our cases we will be setting $\varphi_i = \exp_{z_i}$.

\subsection{Base Distributions}\label{subsec:basedistr}

\textbf{Hyperbolic Space.} We will use the hyperbolic wrapped normal distribution $\mc{G}(\mu, \Sigma)$ where $\mu \in \M$ and $\Sigma \in \R^{n \times n}$ where $\M$ has dimension $n$ \cite{Nagano2019AWN}.

\begin{enumerate}
    \item \textbf{Sampling.} To sample a $m \in \M$, perform the following steps. A priori, set some $\mu_0 \in \M$. First, sample $v \in \mathcal{N}(\mu_0, \Sigma) \in T_{\mu_0}\M$. Then parallel transport this vector to the mean $\mu$ i.e. $u = \mrm{PT}_{\mu_0 \to \mu}(v)$. Lastly, project to the manifold using the exponential map $m = \exp_{\mu}(u)$.
    
    \item \textbf{Probability Density.} The probability density can be found through the composition of the parallel transport map and exponential map. Specifically, we have that
    
    \begin{equation}
        \log p(x) = \log p(v) - \log \abs{\det D_u\exp_\mu(u)} - \log \abs{\det D_v \mrm{PT}_{\mu_0 \to \mu}(v)}
    \end{equation}
\end{enumerate}

\textbf{Spherical Space.} We could possibly perform a relatively similar wrapped normal distribution \cite{Skopek2019MixedcurvatureVA}. However, we see that this is theoretically flawed since parallel transport between two conjugate points is not well defined. 

Instead, we will use the von Mises-Fisher distribution, a distribution on the $(n-1)$-sphere in $\mathbb{R}^n$ originally derived for use in directional statistics \cite{Fisher1987StatisticalAO}. We denote this distribution as $\mrm{vMF}(\mu, \kappa)$ where $\mu \in \mathbb{R}^n$ is treated as an element of $\mathbb{S}^{n-1}$ via the canonical embedding, and $\kappa \in \R_{\ge 0}$. Note the following about the von Mises-Fisher distribution:

\begin{enumerate}
    \item \textbf{Sampling.} The von Mises-Fisher can be sampled from with efficient methods \cite{Ulrich1984VMF,Davidson2018SVAE}.
    
    \item \textbf{Probability Density.} From \cite{Fisher1987StatisticalAO,Davidson2018HypersphericalVA}, we know that the density is given by

    \begin{align}
        p(z) & = \mc C_n(\kappa) \exp(\kappa \mu^T z)\\
        C_n(\kappa) & = \frac{\kappa^{n/2-1}}{(2\pi)^{n/2} \mc I_{n/2-1}(\kappa)}
    \end{align}
    
    Note that $||\mu||^2 = 1$, $\mathcal{C}_n(\kappa)$ is the normalizing constant, and that $\mathcal{I}_v$ denotes the modified Bessel function of the first kind of order $v$.
\end{enumerate}

These baseline probability distributions are visualized in Figure \ref{fig:basedistr}.

\begin{figure}
    \centering
    \subfloat[$\mc G(\mu_0, I)$]{\includegraphics[width=0.15\textwidth]{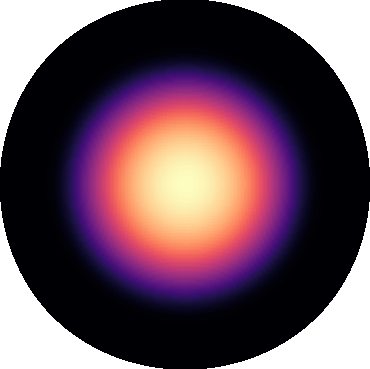}} \qquad
    \subfloat[$\mrm{vMF}(\mu_0,1)$]{\includegraphics[width=0.2\textwidth]{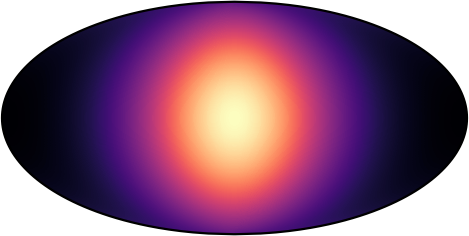}} \qquad
    \caption{Base distributions used for flow models. In $\mathbb H^2$, $\mu_0 = (1,0,0)$ and in $\mathbb S^2$, $\mu_0 = (-1,0,0)$. (a) is on $\mathbb{H}^2$ as visualized on the Poincar\'e ball. (b) is on $\mathbb S^2$ as visualized by the Mollweide projection.}
    \label{fig:basedistr}
\end{figure}

\subsection{Hyperbolic Space}

\subsubsection{Analytic Derivations}

\begin{table}[t]
    \centering
    \caption{Formulas for basic operations in hyperbolic space $\mathbb H^n$.}
    \label{tab:hyperbolicops}
    \begin{tabular}{cc}
    \toprule
        Manifold & $\mathbb{H}^n = \{x \in \R^{n + 1}: \norm{x}_{\mathcal{L}} = -1, x_0 > 0\}$\\[7pt]
        Tangent space & $T_x \mathbb{H}^n = \{v \in \R^{n + 1} : \inn{x, y}_{\mathcal{L}} = 0\}$\\[7pt]
        Exponential map &  $\exp_x(v) = \cosh(\norm{v}_{\mathcal{L}})x + \sinh(\norm{v}_{\mathcal{L}}) \frac{v}{\norm{v}_{\mathcal{L}}}$\\[7pt]
        Logarithmic map & $\log_x(y) = \frac{\arccosh(\inn{x, y}_{\mathcal{L}})}{\sinh(\arccosh(\inn{x, y}_{\mathcal{L}})}(y - \inn{x, y}_{\mathcal{L}} x)$\\[7pt]
        Parallel transport & $\mrm{PT}_{x\to y}(v) = v - \frac{\inn{y, v}_{\mathcal{L}}}{1 + \inn{x, y}_{\mathcal{L}}}(x + y)$\\[7pt]
        Tangent projection & $\mrm{proj}_x(u) = u + \inn{x, u}_{\mathcal{L}} x$ \\[7pt]
    \bottomrule
    \end{tabular}
\end{table}

For hyperbolic space, we will work with the hyperboloid $\mathbb{H}^n$. The analytic values of the operations are given in Table \ref{tab:hyperbolicops}. Recall that the \textit{Lorentz Inner Product and Norm} are given by

\begin{equation}
    \inn{x, u}_{\mathcal{L}} := -x_0u_0 + x_1u_1 + \dots x_nu_n \qquad \norm{x}_{\mathcal{L}} = \sqrt{\inn{x, x}_{\mathcal{L}}}
\end{equation}

In addition, there are many useful identities which appear in our pipeline.

\begin{enumerate}
    \item \textbf{Stereographic Projection.} To visualize $\mathbb H^2$ on the Poincar\'e ball, we use the stereographic projection as explained in  \cite{Skopek2019MixedcurvatureVA}, which maps a point $(\xi, x^T) \in \mathbb H^2 \subseteq \RR^3$ to a the point $x/(1+\xi) \subseteq \RR^2$ on the Poincar\'e ball.
    
    \item \textbf{Log Determinants.} The log determinant of the derivative of the exponential map is given by \cite{Nagano2019AWN, Skopek2019MixedcurvatureVA}:
    
    \begin{equation}\label{eqn:hyperbolicLogdet}
        \log \abs{\det D_v \exp_x} = (n-1) \left[\log \sinh(\norm{v}_{\mc L}) - \log \norm{v}_{\mc L}\right]
    \end{equation}
    
    The log determinant of the derivative of the log map is the negation of the above by the inverse function theorem, and the log determinant of the derivative of parallel transport is $0$.
\end{enumerate}

\subsubsection{Numerical Stability}

In order to ensure numerical stability, we examine several operations which are inherently numerically unstable and present solutions

\begin{enumerate}
    \item \textbf{Arccosh.} Arccosh has a domain of $(-\infty, -1) \cup (1, +\infty)$. In practice, we are concerned with the positive case, although the negative case can be similarly handled. Due to numerical instability a value $1 + \epsilon$ may be realized as $1$ in our floating point system. To compensate, we clamp the minimum value to be $1 + \epsilon_0$ for a small fixed $\epsilon_0$.

    \item \textbf{Sinh Division.} In the exponential and logarithmic maps, there exist terms of the form $\frac{\sinh(x)}{x}$. When $|x| < \epsilon$ for some small $\epsilon$, this is numerically unstable. We special case this (and the derivative) for stability by explicitly deriving the limit value of $x \rightarrow 0$ for these cases.
\end{enumerate}

\subsection{Spherical Space}

\subsubsection{Analytic Derivations}

For spherical space, we work with the sphere $\mathbb{S}^n$. The analytic values are given in Table \ref{tab:sphereops}. Norms and inner products are assumed to be the Euclidean $\ell^2$ values.

\begin{table}[t]
    \centering
    \caption{Formulas for basic operations on the sphere $\mathbb S^n$.}
    \label{tab:sphereops}
    \begin{tabular}{cc}
    \toprule
        Manifold & $\mathbb{S}^n = \{x \in \RR^{n+1} : \norm{x} = 1\}$ \\[7pt]
        Tangent Space & $T_x \mathbb{S}^n = \{v \in \R^{n + 1} : \inn{x, v} = 0 \}$\\[7pt]
        Exponential map &  $\exp_x(y) = \cos(\norm{v}) x + \sin(\norm{v}) \frac{v}{\norm{v}} $ \\[7pt]
        Logarithmic map & $\log_x(y) = \frac{\arccos(\inn{x, y})}{\sin(\arccos(\inn{x, y}))}(y-\inn{x, y} x)$ \\[7pt]
        Parallel transport & $\mrm{PT}_{x\to y}(v) = v - \frac{\inn{y, v}}{1 + \inn{x, y}}(x + y)$ \\[7pt]
        Tangent projection & $\mrm{proj}_x(u) = u - \inn{x, u} x$\\[7pt]
    \bottomrule
    \end{tabular}
\end{table}

Some useful identities are

\begin{enumerate}
    \item \textbf{Mollweide Projection.} To visualize $\mathbb S^2$, we use the Mollweide projection that is used in cartography. For latitude $\theta$ and longitude $\varphi$, the sphere is projected to coordinates $(x,y)$ by (where $\beta$ is a variable solved for by the given equation) \cite{lapaine2011mollweide}:
    \begin{equation}
    2\beta + \sin(2\beta) = \pi \sin(\varphi), \quad x = \frac{2\sqrt{2}}{\pi} \varphi \cos(\beta), \quad y = \sqrt{2} \sin(\beta)
    \end{equation}

    \item \textbf{Log Determinants.} The log determinant for the exponential map of the Sphere is given by
    
    \begin{equation}\label{eqn:sphericalLogdet}
        \log \abs{\det D_v \exp_x} = (n - 1) \sqbrac{\log \sin(\norm{v}) - \log \norm{v}}
    \end{equation}
    
    The log determinant of the derivative of the log map is the negation of the above and the log determinant of the derivative of parallel transport is $0$.
\end{enumerate}

\subsubsection{Numerical Stability}

In order to ensure numerical stability, we note that several functions are inherently numerically unstable

\begin{enumerate}
    \item \textbf{Sine division.} In the exp and log maps, there are values of the form $\frac{\sin x}{x}$. Note that this is numerically instable when $|x| < \epsilon$. We special case this (and its derivative) to allow for better propagation.

    \item \textbf{Log Derivative.} We find that we need an explicit derivation of $D_x \log y$ for our Manifold ODE on the Sphere (see \ref{appendix:nnDesign}). Note that this value can be computed using backpropagation, but we derive it explicitly instead due to numerical instability of the higher order derivatives of some of our functions.
    
    \begin{lemma}\label{lem:dlog}
    For $x, y \in \mc S^n$, and $r= x^T y$, if $|r| \neq 1$, then
    \begin{equation}
        D_y\log_x(y) = \left(\frac{r\arccos(r)}{(1-r^2)^{3/2}} - \frac{1}{1-r^2}\right)\left(y-rx\right)x^T + \frac{\arccos(r)}{\sin(\arccos(r))}\left(I - xx^T \right)
    \end{equation}
    The limit of $D_y\log_x(y)$ as $r = x^T y \to 1$ is $I-xx^T$.
    \end{lemma}
    \begin{proof}
        We differentiate the equation of the logarithmic map as given in Table \ref{tab:sphereops}. First, suppose that $\abs{r} \neq 1$. By the product rule we have,
        \begin{align*}
            D_y\log_x(y) & = D_y \left(\frac{\arccos(r)}{\sin(\arccos(r))}\right) (y- r x) + \frac{\arccos(r)}{\sin(\arccos(r))}D_y(y - r x).
        \end{align*}
        The summand on the right is given by
        \begin{align*}
            \frac{\arccos(r)}{\sin(\arccos(r))} (I - xx^T).
        \end{align*}
        To compute the left summand, we use the chain rule and differentiate by $r$
        \begin{align*}
            D_y \left(\frac{\arccos(r)}{\sin(\arccos(r))}\right) (y- r x) & = \parderiv{}{r} \left(\frac{\arccos(r)}{\sin(\arccos(r))}\right) (y- r x) x^T\\
            & = \left(\frac{r\arccos(r)}{(1-r^2)^{3/2}} - \frac{1}{1-r^2}\right) (y- r x) x^T
        \end{align*}
        To check that the limit as $x^T y \to 1$ is $I-xx^T$, it is enough to compute three separate limits that are all finite. It is clear that
        \[\lim_{y \to x} (y-rx)x^T = 0.\]
        Since the limit of the other term in the left summand can be shown to be finite, this means that the left summand is zero in the limit. For the right summand, the only term that depends on $y$ has a limit
        \[\lim_{r \to 1} \frac{\arccos(r)}{\sin(\arccos(r))} = \lim_{r\to 1} \frac{\arccos(r)}{\sqrt{1-r^2}} =  1\]
        where the final equality can be seen by L'Hopital's rule. Thus, the entire limit is $I-xx^T$.
    \end{proof}
\end{enumerate}

\subsection{Backpropagation}

To update the parameters of an MCNF, we need to differentiate $\log p(\mrm{MODE})$ with respect to $\theta$, so we need to differentiate each of the summands in (\ref{eqn:mcnflog}) with respect to $\theta$. 
Differentiating through neural ODE blocks is done with the Euclidean adjoint method \cite{pontryagin1962mathematical,Chen2018NeuralOD,Grathwohl2019FFJORDFC}, which, for a loss $L$ depending on the solution $\mbf y : [t_s, t_e] \to \RR^n$ to a differential equation with dynamics $g(\mbf y(t), t; \theta)$, gives that
\begin{equation}
    \parderiv{L}{\theta} = -\int_{t_e}^{t_s} \parderiv{L}{\mbf y(t)} \parderiv{g(\mbf y(t), t; \theta)}{\theta} \; dt
\end{equation}
Differentiating the dynamics is done with standard backpropagation. The adjoint state $\parderiv{L}{\mbf y(t)}$ is computed by the solution of the adjoint differential equation (\ref{eqn:manifoldAdjoint}) for Euclidean space, with initial condition $\parderiv{L}{\mbf y(t_e)}$. The derivative of the loss $\parderiv{L}{\mbf y(t_e)}$ can be computed directly. For MCNF, $L$ is taken to be the negative log likelihood.

For the hyperbolic and spherical cases, the log determinant terms take simple forms (as in equations \ref{eqn:hyperbolicLogdet} and \ref{eqn:sphericalLogdet}), and are thus easy to differentiate through. Moreover, for the hyperbolic VAE models, we train by maximizing the evidence lower bound (ELBO) on the log likelihood, so that differentiation is done with a reparameterization as in \cite{Nagano2019AWN}.

\subsection{Designing Neural Networks}\label{appendix:nnDesign}

\subsubsection{Construction}

In general we construct the dynamics $f(\mbf{z}(t), t) \in T_{\mbf{z}(t)}\M$ as follows. Suppose $\M$ is embedded in some $\R^d$. Then we construct $f$ to be a neural network with input of dimension $d + 1$. The first $d$ values are the manifold input and the last value is the time. The output of the neural network will be some vector in $\R^d$. To finalize, we project onto the tangent space using the linear projection $\mrm{proj}_{\mbf{z}(t)}$.

\textbf{Hyperbolic Space.} Since hyperbolic space $\mathbb{H}^n$ is diffeomorphic to Euclidean space, we can parameterize all manifold dynamics with a corresponding Euclidean dynamic with an $\exp_{\mu_0}$, where $\mu_0$ is the point $(1, 0, \dots, 0) \in \R^{n + 1}$. In general since we only require one chart, our Manifold ODE consists of $\exp_0 \circ \mrm{ODE} \circ \log_0$, which means that we can model our full dynamics only in the tangent space (not on the manifold). By picking our basis, we can represent elements of $T_0\M$ as element of $\R^n$. For the ODE block, we parameterize using a neural network $f$ which takes in an input of dimension $n + 1$ (which is a tangent space element and time) and outputs an element of dimension $n$.

\textbf{Spherical Space.}

For the spherical case, we use the default construction (with projection), as there is no global diffeomorphism. Note that when passing from the manifold to tangent space dynamics, we require $D_y \log_x$. We also must invoke a radius of injectivity, as opposed to hyperbolic space. This is $\pi$ for all points. 

\subsubsection{Existence of a Solution}

We construct our networks in such a way that the Picard-Lindel\"of theorem holds. Our dynamics are given by $D_z \varphi \circ f$ where $\varphi$ is a chart and $f$ is a neural network. These are well behaved since 1) the neural network dynamics are well behaved using tanh and other Lipchitz nonlinearities and 2) the chart is well behaved since we can bound the domain to be compact. 

\section{Experimental Details}\label{appendix:trainingDetails}

\subsection{Data}\label{appendix:data}

In our code release, we will include functions to generate samples from the target densities that are used for density estimation in section~\ref{subsec:DensityEstimation}.

\textbf{Hyperbolic Density Estimation} We detail the hyperbolic densities in each row of Figure~\ref{fig:hyperbolicDensity}. 

\begin{enumerate}
    \item The hyperbolic density in the first row of Figure~\ref{fig:hyperbolicDensity} is a wrapped normal distribution with mean $(-1, 1)$ and covariance $\frac{3}{4} I$.
    \item The second density is built from a mixture of 5 Euclidean gaussians, with means $(3,0), (-3,0), (0,3), (0,-3),$ and $(0,0)$, and covariance $\frac{1}{2} I$. The resulting hyperbolic density is obtained by viewing $\RR^2$ as the tangent space $\mc T_0 \mc M$, and then projecting the Euclidean density to the Hyperboloid by $\exp_0$.
    \item The third density is a projection onto the hyperboloid of a uniform checkerboard density in $\mc T_0 \mc M$. The square in the second row and third column of the checkerboard has its lower-left corner at the origin $(0,0)$. Each square has side length $1.5$.
    \item The fourth density is a mixture of four wrapped normal distributions. Letting $s=1.3$, $\sigma_1 = .3$ and $\sigma_2 = 1.5$, the wrapped normals are given as: 
    \begin{align*}
    \mc G\left((0, s, s), \mrm{diag}(\sigma_1, \sigma_2)\right) & \qquad  \mc G\left((0, -s, -s), \mrm{diag}(\sigma_1, \sigma_2)\right)\\
    \mc G\left((0, -s, s), \mrm{diag}(\sigma_2, \sigma_1)\right) & \qquad  \mc G\left((0, s, -s), \mrm{diag}(\sigma_2, \sigma_1)\right)
    \end{align*}
\end{enumerate}

\textbf{Spherical Density Estimation} Details are given about the densities that were learned in each row of Figure~\ref{fig:sphereDensity}.

\begin{enumerate}
    \item The density in the first row of Figure~\ref{fig:sphereDensity} is a wrapped normal distribution with mean $\frac{1}{\sqrt{3}}(-1, -1, -1)$ and covariance $\frac{3}{10} I$.
    \item The second density is built from a mixture of 4 wrapped normals, with means $\frac{1}{\sqrt{3}}(1,1,1), \frac{1}{\sqrt{3}}(-1,-1,-1), \frac{1}{\sqrt{3}}(-1,-1,1)$, and $\frac{1}{\sqrt{3}}(1,1,-1)$; all components of the mixture have covariance $\frac{3}{10} I$. 
    \item The third density is a uniform checkerboard density in spherical coordinates $(\varphi, \theta) \in [0,2\pi] \times [0,\pi]$. The rectangle in the second row and third column of the checkerboard has its lower-left corner at $(\pi, \pi/2)$. The side length of each rectangle in the $\varphi$-axis is $\pi/2-0.2$, and the side length in the $\theta$-axis is $\pi/4-0.1$.
\end{enumerate}

\textbf{Variational Inference} For variational inference, we dynamically binarize the MNIST and Omniglot images with the same procedure as given in \cite{Skopek2019MixedcurvatureVA}. We resize the Omniglot images to $28 \times 28$, the same size as the MNIST images.

\subsection{Density Estimation}

In section~\ref{subsec:DensityEstimation} we train on batches of 200 samples from the target density (or batches of size 100 for the discrete spherical normalizing flows \cite{Rezende2020NormalizingFO}). Our MCNF models and the hyperbolic baselines use at most 1{,}000{,}000 samples, with early stopping as needed---the hyperbolic baselines sometimes diverge when training for too many batches. As suggested in \cite{Rezende2020NormalizingFO}, we find that the spherical baseline does indeed needed more samples than this (at least 5{,}000{,}000), so we allow it to train until the density converges. Although we do not investigate sample efficiency in detail, we find that our MCNF is able to achieve better results than the spherical baseline with, frequently, over an order of magnitude fewer samples than the spherical baseline. Note that all methods use the Adam optimizer \cite{Kingma2014adam}. For our MCNF, our dynamics are given by a neural network of hidden dimension 32 and 4 linear layers with tanh activation; for each integration we use a Runge-Kutta 4 solver.

\textbf{Hyperbolic Normalizing Flows} For the hyperbolic discrete normalizing flows, we train with 4 hidden blocks, hidden flow dimension of 32, and tanh activations. The prior distributions used are given in section~\ref{subsec:basedistr} and target distributions are given in section~\ref{appendix:data}.

\textbf{Spherical Normalizing Flows} For the discrete spherical normalizing flows from \cite{Rezende2015VariationalIW}, we use the recursive flow for $\mathbb{S}^2$. In designing this flow, we use the non-compact projection (NCP) flow for $\mathbb{S}^1$ and the autoregressive spline flow from \cite{Durkan2019NeuralSF} for the interval $[-1,1]$. To increase expressiveness for the $\mathbb{S}^1$ flow, we consider learning a convex combination of NCP flows over the circle. Let the number of components in this combination be $n$. To increase the expressiveness of the spline flow over $[-1,1]$, we increase the number of segments. Let the number of segments be $k$. We tuned $n$ and $k$ for each of the $3$ spherical densities to yield the best results. Note that the best $n$ and $k$ frequently ended up being fairly small for the more simple densities, since having an overly expressive model for simple densities ended up being hard to train and produced undesirable artifacts.

The prior distributions used for all target distributions are given in section~\ref{subsec:basedistr}. For the first spherical target distribution in section~\ref{appendix:data}, we use $k=2$ and $n=1$. For the second spherical target distribution we use $k=6$ and $n=2$. For the final spherical target distribution we use $k=32$ and $n=12$.

\subsection{Variational Inference}\label{appendix:detailsVAE}

In each model for variational inference, we train with a hyperbolic or Euclidean VAE. Following \cite{Bose2020LatentVM}, in both cases, the mean and variance encoders are taken to be one layer neural networks with a hidden dimension of 600. The $\mrm{ReLU}$ nonlinearity is used for the hyperbolic VAE and the $\mrm{LeakyReLU}$ is used for the Euclidean VAE. As is often done, we take the weight matrices of the first linear layer to be shared for the mean and variance \cite{Bose2020LatentVM,Kingma2014AutoEncodingVB}. The decoder is taken to have a symmetric architecture, with the input coming from the latent space and the output being a decoded image. We vary the latent dimension from 2 to 6 in our experiments. When we add a discrete normalizing flow, we use 2 hidden blocks and 2 hidden dimensions per block with a hidden layer of size 128 and tanh activations, again following \cite{Bose2020LatentVM}. For continuous flows, we replace this with a neural ODE or MCNF block where the dynamics are parameterized by a two-layer network of hidden size 128 with tanh activations. For numerical integration we use the Runge-Kutta 4 solver.

\section{Dynamic Chart Method}\label{appendix:dynamicChart}

Here we elaborate on the dynamic chart method presented in Section \ref{sec:trivialization} of our paper. Specifically, we discuss the significance of dynamic charts and the generality of the multi-chart MCNF (which allows learning of arbitrary densities on manifolds with conjugate points, like $\mathbb{S}^2$).

\subsection{Choosing Charts}

\textbf{The Benefit of Dynamic Charts.} Note that our dynamic chart trivialization is performed in the main paper simply by splitting the time interval $[t_s, t_e]$ up uniformly into segments of length $\epsilon$ and learning the solution locally via $\exp$-map charts at the anchor points (endpoints of the segments). Such a splitting allows for the \textit{ball of injectivity} (induced by the radius of injectivity) to ``move" throughout the training process, so that it always surrounds the locally relevant region (centered at the anchor point). This approach allows for transportation of mass that evades the conjugate point problem (which \cite{Gemici2016NormalizingFO} does not allow for).

This becomes especially clear if we consider the case of conjugate points on a sphere, i.e. the case of antipodal points. Consider a task in which we have a prior with mass surrounding one antipodal point and the target density is concentrated around the other point. A one chart approach would have trouble transporting the mass due to the fact that a fixed $\exp$-map can never have a ball of injectivity (induced by the radius of injectivity $\pi$, in this case) that includes both points throughout training. However, our dynamic approach allows this ball of injectivity to shift throughout the training process and makes correct transportation of mass for such a scenario possible. An experiment testifying to this is given in Appendix \ref{sec:multchartexample} (and Figure \ref{fig:dynamicChartStability}).

\textbf{Non-uniform Time-domain Segmenting.} While our approach allows for the ball of injectivity to shift, the dynamic chart method is not limited to a uniform $\epsilon$-segment splitting of the $[t_s, t_e]$ interval. Similar benefits may be derived with an alternative splitting. However, it is not clear what additional benefits a non-uniform splitting might bring without prior knowledge of local manifold topology. Our experiments find that a uniform split works well for the densities we tested.

\textbf{Control of Local Dynamics} Note that although the dynamic chart approach makes it possible for the ball of injectivity to move (making it possible to learn general densities), local dynamics may still cause issues. This is because the local dynamics of the ODE may cause the solver to venture to the edges of the ball of injectivity surrounding the current anchor point $x$, thereby causing instability. One may resolve this issue by enforcing a Lipschitz constraint on the solution by explicitly bounding the derivative (e.g. bounding $\frac{d\mathbf{y}(t)}{dt}$). If we call the Lipschitz constant $L \in \mathbb{R}$, the length of chart domain $\epsilon > 0$, and the radius of injectivity $r_x$ at the current anchor point $x$, we would want to enforce the constraint such that $\epsilon L < r_x$, i.e. $L < r_x/\epsilon$. Notice that this is only a local Lipschitz constraint, since it depends on the radius of injectivity at each anchor point, and moreover, on the way the time-domain segmenting is done. Enforcing this constraint, even in a simplistic way (e.g. thresholding $\frac{d\mathbf{y}(t)}{dt}$ at $r_x/\epsilon$), ensures that the dynamics are configured so the local solution stays within the ball of injectivity and instability is avoided. We note that for most of our experiments, we did not worry about this, as our method was already stable, but we include this clarification in the case that it becomes necessary to maintain stability (for e.g. a particularly complicated density on a manifold with conjugate points).

\subsection{Expressivity and Generality of MCNF}
\label{sec:multchartexample}

\begin{figure}[t]
    \centering
    \subfloat[Target]{\includegraphics[width=.4\textwidth]{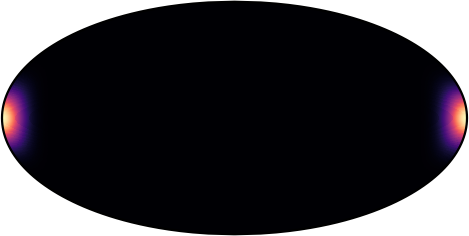}}\\
    \begin{tabular}{cc}
    \subfloat[1 chart]{\includegraphics[width=.4\textwidth]{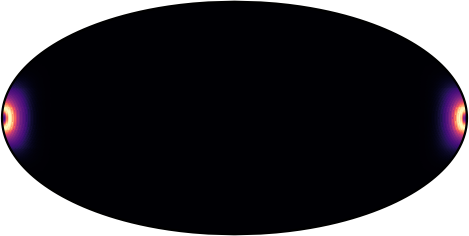}} &
    \subfloat[16 charts]{\includegraphics[width=.4\textwidth]{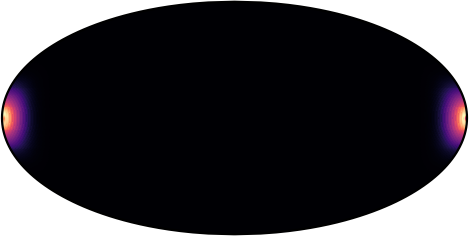}}
    \end{tabular}
    \caption{Comparison of MCNF for different numbers of charts. Note that using just 1 chart is not enough to learn the density concentrated at the antipode, while the 16 chart model learns the density well.}
    \label{fig:dynamicChartStability}
\end{figure}

To demonstrate the effectiveness of our dynamic chart method in getting around conjugate points and numerical instability around them, we set up a particular density on the sphere for estimation. The target density is a $\mrm{vMF}((1,0,0), 30)$ distribution, which is heavily concentrated around $(1,0,0)$, as shown in Figure \ref{fig:dynamicChartStability} (a). We still take our base distribution as the vMF centered at $\mu_0 = (-1,0,0)$ (see Figure \ref{fig:basedistr}), but now we take it to be more concentrated by setting $\kappa = 3$. To learn the target density using this base distribution, our MCNF must learn to move probability density from samples around $(1,0,0)$ to areas of high base probability density around the antipodal point $(-1,0,0)$.

As shown in Figure~\ref{fig:dynamicChartStability}, MCNF with just one chart is not able to learn a density with high probability at $(1,0,0)$, as it is unable to move samples around $(1,0,0)$ close enough to the antipodal point. On the other hand, MCNF with 16 charts is able to do so, thus validating our model's ability to get around conjugate points with the dynamic chart method.

\section{Generated MNIST Samples}\label{appendix:samples}

Here, we generate samples from MNIST using our trained hyperbolic MCNF. To do this, we use an MCNF with the same setup as in \ref{appendix:detailsVAE}, except with a slightly larger VAE architecture. We add a linear layer to both the mean and variance networks, add an additional shared linear layer for both of them, and add a linear layer to the decoder. MNIST samples generated with our approach are given in Figure \ref{fig:mnistsamples}. With a latent space of dimension 2, the MCNF generates examples that resemble real digits. Interpolating between points in the latent space gives hybrid intermediaries that meaningfully represent semantic change (for instance, going between a ``4" and a ``7" produces instances of ``9").

\begin{figure}[h]
    \centering
    \subfloat[]{\includegraphics[width=.45\textwidth]{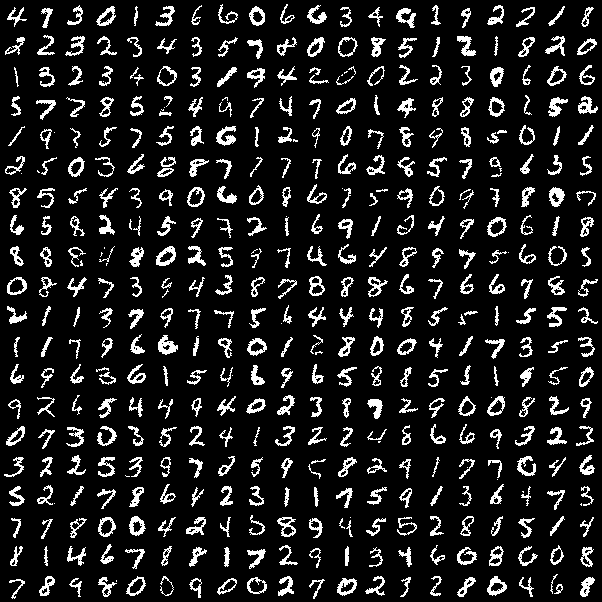}}\\
    \begin{tabular}{cc}
    \subfloat[]{\includegraphics[width=.45\textwidth]{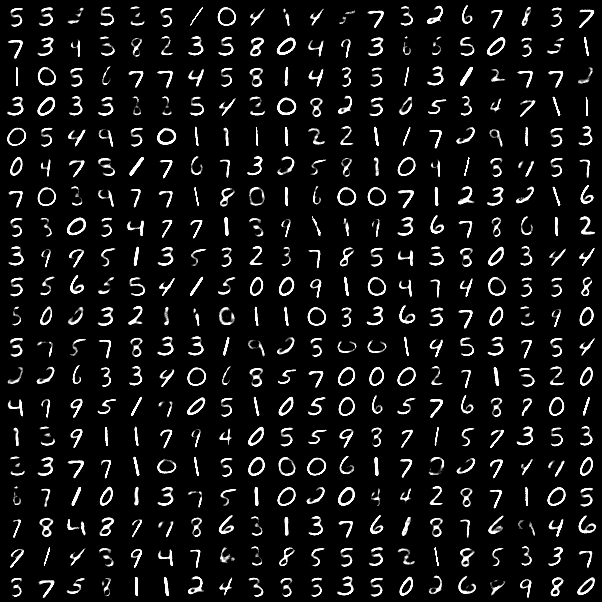}} &
    \subfloat[]{\includegraphics[width=.45\textwidth]{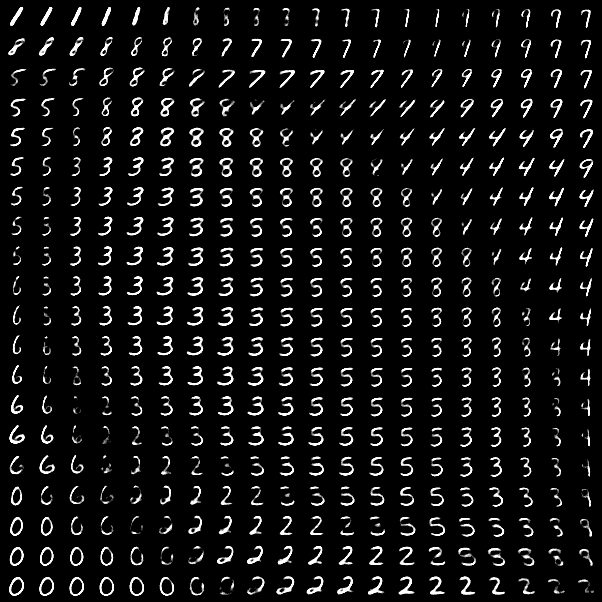}} 
    \end{tabular}
    \caption{(a) Real sample images from (Binarized) MNIST. (b) Random generated samples from a hyperbolic MCNF trained on Binarized MNIST with latent dimension 2. (c) Generated samples from the same trained MCNF, where the latent variables are taken from the projection onto $\mathbb H^2$ of a uniformly spaced grid on the tangent space $\mc T_0 \mathbb H^2$. The generated samples are visualized in their corresponding positions on $\mc T_0 \mathbb H^2$.}
    \label{fig:mnistsamples}
\end{figure}

\end{document}